\documentclass[10pt]{article}
\usepackage[usenames]{color} 
\usepackage{amssymb} 
\usepackage{amsmath} 
\usepackage{amsthm}
\usepackage{nips12submit_e,times}
\usepackage{graphicx}

\usepackage{paralist}

\def\B{\mathcal{B}}

\def\R{\mathbb{R}}
\def\E{\mathbb{E}}

\def\F{\mathcal{F}}

\def\P{\mathbb{P}}

\def\1{\mathbf{1}}

\def\sign{\mbox{sign}}

\newtheorem{theorem}{Theorem}
\newtheorem{lemma}{Lemma}

\title{Query Complexity of Derivative-Free Optimization}

\author{Kevin G. Jamieson\\
University of Wisconsin\\
Madison, WI 53706, USA \\
\texttt{\small kgjamieson@wisc.edu}
\And 
Robert D. Nowak\\
University of Wisconsin\\
Madison, WI 53706, USA \\
\texttt{\small nowak@engr.wisc.edu}
\And
Benjamin Recht\\
University of Wisconsin\\
Madison, WI 53706, USA \\
\texttt{\small brecht@cs.wisc.edu}
}

\nipsfinalcopy 

\begin{document}

\maketitle

\begin{abstract}
This paper provides lower bounds on the convergence rate of Derivative Free Optimization (DFO) with noisy function evaluations, exposing a fundamental and unavoidable gap between the performance of algorithms with access to gradients and those with access to only function evaluations. However, there are situations in which DFO is unavoidable, and for such situations we propose a new DFO algorithm that is proved to be near optimal for the class of strongly convex objective functions.  A distinctive feature of the algorithm is that it uses only Boolean-valued function comparisons, rather than function evaluations.  This makes the algorithm useful in an even wider range of applications, such as optimization based on paired comparisons from human subjects, for example.  We also show that regardless of whether DFO is based on noisy function evaluations or Boolean-valued function comparisons, the convergence rate is the same.
\end{abstract}

\section{Introduction}

Optimizing large-scale complex systems often requires the tuning of many parameters. With training data or simulations one can evaluate the relative merit, or incurred loss, of different parameter settings, but it may be unclear how each parameter influences the overall objective function. In such cases, derivatives of the objective function with respect to the parameters are unavailable. Thus, we have seen a resurgence of interest in Derivative Free Optimization (DFO) \cite{eitrich,oeuvray,conn,PowellLearning, nesterov,srinivas,storn,agarwal2011stochastic}. When function evaluations are noiseless, DFO methods can achieve the same rates of convergence as noiseless gradient methods up to a small factor depending on a low-order polynomial of the dimension \cite{nemirovski,nesterov, protasov}. This leads one to wonder if the same equivalence can be extended to the case when function evaluations and gradients are noisy.

Sadly, this paper proves otherwise.  We show that when function evaluations are noisy,  the optimization error of \emph{any} DFO is  $\Omega(\sqrt{1/T})$, where $T$ is the number of evaluations.  This lower bound holds even for strongly convex functions.  In contrast, noisy gradient methods exhibit $\Theta(1/T)$ error scaling for strongly convex functions \cite{nemirovski,raginskyConvexComplexity}.  A consequence of our theory is that finite differencing cannot achieve the rates of gradient methods when the function evaluations are noisy.

On the positive side, we also present a new derivative-free algorithm that achieves this lower bound with near optimal dimension dependence. Moreover, the algorithm uses only boolean comparisons of function values, not actual function values.  This makes the algorithm applicable to situations in which the optimization is only able to probably correctly decide if the value of one configuration is better than the value of another.   This is especially interesting in optimization based on human subject feedback, where paired comparisons are often used instead of numerical scoring.  The convergence rate of the new algorithm is optimal in terms of $T$ and near-optimal in terms of its dependence on the ambient dimension. Surprisingly, our lower bounds show that this new algorithm that uses only function comparisons achieves the same rate in terms of $T$ as any algorithm that has access to function evaluations. 

\section{Problem formulation and background} \label{problemForm}
We now formalize the notation and conventions for our analysis of DFO. A function $f$ is {\em strongly convex with constant $\tau$} on a convex set $\B\subset \R^d$ if there exists a constant $\tau >0$ such that  
\begin{align*} 
f(y) \geq f(x) + \langle \nabla f(x),y-x \rangle + \frac{\tau}{2} ||x-y||^2
\end{align*}
for all $x,y \in \B$.  The gradient of $f$, if it exists, denoted $\nabla f$, is {\em Lipschitz with constant $L$} if $|| \nabla f(x)-\nabla f(y) || \leq L ||x-y||$ for some $L>0$.  The class of strongly convex functions with Lipschitz gradients defined on a nonempty, convex set $\B \subset \R^n$ which take their minimum in $\B$ with parameters $\tau$ and $L$ is denoted by $\F_{\tau,L,\B}$. 

The problem we consider is minimizing a function $f\in\F_{\tau,L,\B}$. The function $f$ is not explicitly known.  An optimization procedure may only query the function in one of the following two ways.
\begin{description}
\item \textbf{Function Evaluation Oracle:} For any point $x\in\B$ an optimization procedure can observe
\begin{align*}
E_f(x) = f(x) + w
\end{align*}
where $w\in\R$ is a random variable with $\E[w]=0$ and $\E[ w^2] = \sigma^2$. 
\item \textbf{Function Comparison Oracle:} For any pair of points $x,y\in\B$ an optimization procedure can observe a binary random variable $C_f(x,y)$ satisfying
\begin{align} \label{funcCompModel}
\P\left( C_f(x,y) =\sign\{ f(y)-f(x) \} \right)  \geq \frac{1}{2}+ \min\left\{ \delta_0,  \mu |f(y)-f(x)|^{\kappa-1} \right\}
\end{align}
for some $0<\delta_0\leq 1/2$, $\mu > 0$ and $\kappa\geq 1$. When $\kappa=1$, without loss of generality assume $\mu \leq \delta_0 \leq 1/2$.  Note $\kappa=1$  implies that the comparison oracle is correct with a probability that is greater than 1/2 and independent of $x,y$.  If $\kappa>1$, then the oracle's reliability decreases as the difference between $f(x)$ and $f(y)$ decreases.
\end{description}

To illustrate how the function comparison oracle and function evaluation oracles relate to each other, suppose $C_f(x,y) = \sign\{ E_f(y)-E_f(x) \}$ where $E_f(x)$ is a function evaluation oracle with additive noise $w$. If $w$ is Gaussian distributed with mean zero and variance $\sigma^2$ then $\kappa =2$ and $\mu \geq \left(4\pi \sigma^2 e\right)^{-1/2}$ (see Appendix~\ref{distributionParams}). In fact, this choice of $w$ corresponds to Thurston's law of comparative judgment which is a popular model for outcomes of pairwise comparisons from human subjects \cite{thurstone}. If $w$ is a ``spikier'' distribution such as a two-sided Gamma distribution with shape parameter in the range of $(0,1]$ then all values of $\kappa \in (1,2]$ can be realized (see Appendix~\ref{distributionParams}). 

Interest in the function comparison oracle is motivated by certain popular derivative-free optimization procedures that use only  comparisons of function evaluations (e.g. \cite{storn}) and by optimization problems involving human subjects making paired comparisons (for instance, getting fitted for prescription lenses or a hearing aid where unknown parameters specific to each person are tuned with the familiar queries ``better or worse?''). Pairwise comparisons have also been suggested as a novel way to tune web-search algorithms \cite{yueDuelingBandits}. Pairwise comparison strategies have previously been analyzed in the finite setting where the task is to identify the best alternative among a finite set of alternatives (sometimes referred to as the dueling-bandit problem) \cite{yueDuelingBandits,activeRanking}.  The function comparison oracle presented in this work and its analysis are novel. The main contributions of this work and new art are as follows (i) lower bounds for the function evaluation oracle in the presence of measurement noise (ii) lower bounds for the function comparison oracle in the presence of noise and (iii)  an algorithm for the function comparison oracle, which can also be applied to the function evaluation oracle setting, that nearly matches both the lower bounds of (i) and (ii).

We prove our lower bounds for strongly convex functions with Lipschitz gradients defined on a compact, convex set $\B$, and because these problems are a subset of those involving all convex functions (and have non-empty intersection with problems where $f$ is merely Lipschitz), the lower bound also applies to these larger classes. While there are known theoretical results for DFO in the noiseless setting \cite{nemirovskyYudin,nesterov,protasov}, to the best of our knowledge we are the first to characterize lower bounds for DFO in the stochastic setting. Moreover, we believe we are the first to show a novel upper bound for stochastic DFO using a function comparison oracle (which also applies to the function evaluation oracle). However, there are algorithms with upper bounds on the rates of convergence for stochastic DFO with the function evaluation oracle \cite{nemirovskyYudin,agarwalBanditDFO}. We discuss the relevant results in the next section following the lower bounds . 

While there remains many open problems in stochastic DFO (see Section~\ref{conclusion}), rates of convergence with a stochastic gradient oracle are well known and were first lower bounded by Nemirovski and Yudin \cite{nemirovskyYudin}. These classic results were recently tightened to show a dependence on the dimension of the problem \cite{alekh}. And then tightened again to show a better dependence on the noise \cite{raginskyConvexComplexity} which matches the upper bound achieved by stochastic gradient descent \cite{nemirovski}. The aim of this work is to start filling in the knowledge gaps of stochastic DFO so that it is as well understood as the stochastic gradient oracle. Our bounds are based on simple techniques borrowed from the statistical learning literature that use natural functions and oracles in the same spirit of \cite{raginskyConvexComplexity}. 

\section{Main results}
The results below are presented with simplifying constants that encompass many factors to aid in exposition. Explicit constants are given in the proofs in Sections~\ref{lower}~and~\ref{upper}.  Throughout, we denote the minimizer of $f$ as $x_f^*$.  The expectation in the bounds is with respect to the noise in the oracle queries and (possible) optimization algorithm randomization.
\subsection{Query complexity of the function comparison oracle}
\begin{theorem} \label{comparisonThm}
For every $f \in \F_{\tau,L,\B}$ let $C_f$ be a function comparison oracle  with parameters $(\kappa,\mu,\delta_0 )$. Then for $n \geq 8$ and sufficiently large $T$
\begin{align*}
\inf_{\widehat{x}_T} \sup_{f\in \F_{\tau,L,\B} } \E \left[ f(\widehat{x}_T) - f(x_f^*) \right] \geq \begin{cases}  c_1 \exp\left\{ -c_2 \frac{T}{n}  \right\}  & \text{ if } \kappa=1 \vspace{.1in} \\c_3 \left( \frac{n}{T} \right)^{\frac{1}{ 2 (\kappa-1)}} & \text{ if } \kappa > 1 \end{cases}
\end{align*}
{ where the infimum is over the collection of all possible estimators of $x_f^*$ using at most $T$ queries to a function comparison oracle and the supremum is taken with respect to all problems in $ \F_{\tau,L,\B}$ and function comparison oracles with parameters $(\kappa,\mu,\delta_0 )$.}  The constants  $c_1,c_2,c_3$ depend the oracle and function class parameters, as well as the geometry of $\B$,  but are independent of $T$ and $n$. 
\end{theorem}

{ For upper bounds we propose a specific algorithm based on coordinate-descent in Section~\ref{upper} and prove the following theorem for the case of unconstrained optimization, that is, $\B = \R^n$.}
\begin{theorem} \label{comparisonThmUp}
For every $f \in \F_{\tau,L,\B}$ { with $\B = \R^n$} let $C_f$ be a function comparison oracle with parameters $(\kappa,\mu,\delta_0)$. Then there exists a coordinate-descent algorithm that is adaptive to unknown $\kappa \geq 1$ that outputs an estimate $\widehat{x}_T$ after $T$  function comparison queries such that with probability $1-\delta$
\begin{align*}
  \sup_{f \in\F_{\tau,L,B}} \E \left[  f(\widehat{x}_T) - f(x_f^*) \right] \leq \begin{cases}   c_1\exp\left\{ -c_2  \sqrt{ \frac{T}{n }  }\right\}  & \text{ if } \kappa = 1 \vspace{.1in} \\ c_3  n\left( \frac{n }{T} \right)^{\frac{1}{ 2 (\kappa-1)}} & \text{ if } \kappa > 1 \end{cases}
\end{align*}
where $c_1,c_2,c_3$ depend the oracle and function class parameters as well as $T$,$n$, and $1/\delta$, but only poly-logarithmically.  
\end{theorem}

\subsection{Query complexity of the function evaluation oracle}
\begin{theorem} \label{evaluationThm}
For every $f \in \F_{\tau,L,\B}$ let $E_f$ be a function evaluation oracle with variance $\sigma^2$. Then for $n \geq 8$ and sufficiently large $T$ 
\begin{align*}
\inf_{\widehat{x}_T} \sup_{f\in \F_{\tau,L,\B} } \E \left[ f(\widehat{x}_T) - f(x_f^*) \right] \geq c \left( \frac{n \sigma^2}{T} \right)^{\frac{1}{ 2 }} 
\end{align*}
{ where the infimum is taken with respect to the collection of all possible estimators of $x_f^*$ using just  $T$ queries to a function evaluation oracle and the supremum is taken with respect to all problems in $ \F_{\tau,L,\B}$ and function evaluation oracles with variance $\sigma^2$.} The constant $c$ depends on the oracle and function class parameters, as well as the geometry of $\B$, but is independent of $T$ and $n$. \end{theorem}

{  
Because a function evaluation oracle can always be turned into a function comparison oracle (see discussion above), the algorithm and upper bound in Theorem~2 with $\kappa=2$ applies to many typical function evaluation oracles (e.g. additive Gaussian noise), yielding an upper bound of $  \left( {n^3 \sigma^2}/{T} \right)^{{1}/{ 2 }}$ ignoring constants and log factors. This matches the rate of convergence as a function of $T$ and $\sigma^2$, but has worse dependence on the dimension $n$.

Alternatively, under a less restrictive setting, Nemirovski and Yudin proposed two algorithms for the class of convex, Lipschitz functions that obtain rates of $n^{1/2} / T^{1/4}$ and $p(n) / T^{1/2}$, respectively,  where $p(n)$ was left as an unspecified polynomial of $n$   \cite{nemirovskyYudin}. While focusing on stochastic DFO with bandit feedback, {Agarwal \em{et. al.}} built on the ideas developed in  \cite{nemirovskyYudin} to obtain a result that they point out implies a convergence rate of $n^{16}/ T^{1/2}$ in the optimization setting considered here \cite{agarwalBanditDFO}. Whether or not these rates can be improved to those obtained under the more restrictive function classes of above is an open question. 
}

A related but fundamentally different problem that is somewhat related with the setting considered in this paper is described as online (or stochastic) convex optimization with multi-point feedback  \cite{agarwalMultiPointBandit,nesterov,ghadimi}. Essentially, this setting allows the algorithm to probe the value of the function $f$ plus noise at multiple locations where the noise changes at each time step, but each set of samples at each time experiences the {\em same} noise. Because the noise model of that work is incompatible with the one considered here, no comparisons should be made between the two.

\section{Lower Bounds} \label{lower}
The lower bounds in Theorems 1 and 3 are proved using a general minimax bound \cite[Thm.~2.5]{tsybakov}.  Our proofs are most related to the approach developed in \cite{castroNowak} for active learning, which like optimization involves a Markovian sampling process.   Roughly speaking, the lower bounds are established by considering a simple case of the optimization problem in which the global minimum is known a priori to belong to a finite set.  Since the simple case is ``easier'' than the original optimization, the minimum number of queries required for a desired level of accuracy in this case yields a lower bound for the original problem.  

The following theorem is used to prove the bounds.  In the terms of the theorem, $f$ is a function to be minimized and $P_f$ is the probability model governing the noise associated with queries when $f$ is the true function.
\begin{theorem} \cite[Thm.~2.5]{tsybakov}
Consider a class of functions $\F$ and an associated family of probability measures $\{P_f\}_{f\in\F}$. Let $M\geq 2$ be an integer and $f_0,f_1,\dots,f_M$ be functions in $\F$. Let $d( \cdot , \cdot ) : \F \times \F \rightarrow \R$ be a semi-distance and assume that:
\begin{compactenum}
\item $d(f_i,f_j) \geq 2 s > 0$, for all $0 \leq i<j \leq M$,
\item $\frac{1}{M} \sum_{j=1}^M \mbox{KL}(P_i || P_0 ) \leq a \log M$,
\end{compactenum}
where the Kullback-Leibler divergence $\mbox{KL}(P_i || P_0 ) := \int \log \frac{dP_i}{dP_0} dP_i$ is assumed to be well-defined
(i.e., $P_0$ is a dominating measure) and  $0 < a < 1/8$ . Then
\begin{align*}
\inf_{\widehat{f}} \sup_{f\in\F}\P ( d(\widehat{f},f) \geq s ) \ \geq \ \inf_{\widehat{f}} \max_{f \in \small\{ f_0, \dots, f_M\}}\P ( d(\widehat{f},f) \geq s ) \ \geq \ \textstyle\frac{\sqrt{M}}{1+\sqrt{M}} \left( 1- 2a - 2 \sqrt{\frac{a}{\log M}} \right) > 0 \, ,
\end{align*}
where the infimum is taken over all possible estimators based on a sample from $P_f$.
\label{t-thm}
\end{theorem}

{We are concerned with the functions in the class $\F:=\F_{\tau,L,\B}$. The volume of $\B$ will affect only constant factors in our bounds, so we will simply denote the class of functions by $\F$ and refer explicitly to $\B$ only when necessary. Let $x_f :=\arg \min_x f(x)$, for all $f\in\F$. The semi-distance we use is $d(f,g):=\|x_f-x_g||$, for all $f,g\in \F$. Note that each point in $\B$ can be specified by one of many $f\in\F$.  So the problem of selecting an $f$ is equivalent to selecting a point $x\in\B$.  Indeed, the semi-distance defines a collection of equivalence classes in $\F$ (i.e., all functions having a minimum at $x\in\B$ are equivalent).   For every $f\in \F$ we have $\inf_{g\in \F}f(x_g) = \inf_{x\in\B} f(x)$, which is a useful identity to keep in mind.}

We now construct the functions $f_0,f_1,\dots,f_M$ that will be used for our proofs. Let $\Omega = \{-1,1\}^n$ so that each $\omega \in \Omega$ is a vertex of the $d$-dimensional hypercube. {Let $\mathcal{V} \subset \Omega$ with cardinality $|\mathcal{V}|\geq 2^{n/8}$  such that  for all $\omega\neq\omega' \in \mathcal{V}$, we have $\rho(\omega,\omega') \geq n/8$ where $\rho(\cdot ,\cdot )$ is the Hamming distance.  It is known that such a set exists by the Varshamov-Gilbert bound \cite[Lemma~2.9]{tsybakov}.  Denote the elements of $\mathcal{V}$ by $\omega_0,\omega_1,\dots,\omega_M$. Next we state some elementary bounds on the functions that will be used in our analysis.  }
\begin{lemma} \label{baseFns}
For $\epsilon>0$ define the set $\B \subset \R^n$ to be the $\ell_\infty$ ball of radius $\epsilon$ and define the functions on $\B$: ${f}_i(x) := \frac{\tau}{2}||x-\epsilon \omega_i ||^2$, for $i=0,\dots,M$, $\omega_i \in \mathcal{V}$, and $x_i := \arg \min_x {f}_i(x) = \epsilon \omega_i$. Then for all $0 \leq i < j \leq M$ and $x \in \B$ the functions $f_i(x)$ satisfy
\begin{compactenum}
\item $f_i$ is strongly convex-$\tau$ with Lipschitz-$L$ gradients and $x_i \in \B$ 
\item $|| x_i - x_j|| \geq \epsilon \sqrt{ \frac{n}{2}} $
\item $\displaystyle|f_i(x)-f_j(x)| \leq  2 \tau n\epsilon^2 $ .
\end{compactenum}
\end{lemma}

We are now ready to prove Theorems 1 and 3. Each proof uses the functions $f_0,\dots,f_M$ a bit differently, and since the noise model is also different in each case, the KL divergence is bounded differently in each proof. We use the fact that if $X$ and $Y$ are random variables distributed according to Bernoulli distributions $P_X$ and $P_Y$ with parameters $1/2+\mu$ and $1/2 - \mu$, then $\mbox{KL}(P_X||P_Y) \leq 4\mu^2/(1/2-\mu)$. Also, if $X\sim\mathcal{N}( \mu_X, \sigma^2 )=:P_X$ and $Y\sim\mathcal{N}( \mu_Y, \sigma^2 )=:P_y$ then  $\mbox{KL}(P_X||P_Y) = \frac{1}{2\sigma^2} || \mu_X - \mu_Y||^2$.

\subsection{Proof of Theorem \ref{comparisonThm}} \label{kappaG1}
First we will obtain the bound for the case  $\kappa>1$.
Let the comparison oracle satisfy
\begin{align*}
\P\left( C_{f_i}(x,y) = \sign\{f_i(y)-f_i(x)\}\right) \ = \ \frac{1}{2}+ \min\left\{  \mu |f_i(y)-f_i(x)|^{\kappa-1} , \delta_0\right\} .
\end{align*}
In words, $C_{f_i}(x,y)$ is correct with probability as large as the right-hand-side of above and is monotonic increasing in $f_i(y)-f_i(x)$.  Let $\{x_k,y_k\}_{k=1}^T$ be a sequence of $T$ pairs in $\B$ and let $\{C_{f_i}(x_k,y_k)\}_{k=1}^T$ be the corresponding sequence of noisy comparisons.  We allow the sequence $\{x_k,y_k\}_{k=1}^T$ to be generated in any way subject to the Markovian assumption that $C_{f_i}(x_k,y_k)$ given $(x_k,y_k)$ is conditionally independent of $\{x_i,y_i\}_{i<k}$.
For $i=0,\dots,M$, and $\ell=1,\dots,T$ let $P_{i,\ell}$ denote the joint probability distribution of $\{x_k,y_k,C_{f_i}(x_k,y_k)\}_{k=1}^\ell$, let $Q_{i,\ell}$ denote the conditional distribution of $C_{f_i}(x_\ell,y_\ell)$ given $(x_\ell,y_\ell)$, and let $S_{\ell}$ denote the conditional distribution of $(x_\ell,y_\ell)$ given $\{x_k,y_k,C_{f_i}(x_k,y_k)\}_{k=1}^{\ell-1}$.    Note that $S_{\ell}$ is only a function of the underlying optimization algorithm and does not depend on $i$.
\begin{align*}
\mbox{KL}(&P_{i,T}|| P_{j,T} ) = \E_{P_{i,T}} \left[ \log  \frac{P_{i,T}}{P_{j,T}} \right] = \E_{P_{i,T}} \left[ \log \frac{\prod_{\ell=1}^T Q_{i,\ell} S_{\ell}}{\prod_{\ell=1}^T Q_{j,\ell} S_{\ell}}\right] = \E_{P_{i,T}} \left[ \log \frac{\prod_{\ell=1}^T Q_{i,\ell} }{\prod_{\ell=1}^T Q_{j,\ell} }\right] \\
& = \sum_{\ell=1}^T \E_{P_{i,T}} \left[ \E_{P_{i,T}} \left[ \log \frac{ Q_{i,\ell} }{ Q_{j,\ell} } \bigg | \{x_k,y_k\}_{k=1}^T\right]\right]  \leq T \sup_{x_1,y_1\in \B} \E_{P_{i,1}} \left[ \E_{P_{i,1}} \left[ \log \frac{ Q_{i,1} }{ Q_{j,1} } \bigg | x_1,y_1\right]\right]
\end{align*}
By the second claim of Lemma~\ref{baseFns}, $  |f_i(x)-f_j(x)| \leq   2 \tau n  \epsilon^2 $, and therefore the bound above is less than or equal to the KL divergence between the Bernoulli distributions with parameters $\frac{1}{2}\pm \mu \left( 2 \tau n \epsilon^2  \right)^{(\kappa-1)} $, yielding the bound
\begin{align*}
\mbox{KL}(P_{i,T} | P_{j,T} )  \leq \frac{ 4 T \mu^2 \left( 2 \tau n  \epsilon^2 \right)^{2(\kappa-1)}}{1/2-\mu \left( 2 \tau n \epsilon^2   \right)^{(\kappa-1)} } \ \leq   16 T \mu^2 \left( 2 \tau n \epsilon^2   \right)^{2(\kappa-1)}  
\end{align*}
provided $\epsilon$ is sufficiently small. We also assume $\epsilon$ (or, equivalently, $\B$) is sufficiently small so that $|f_i(x)-f_j(x)|^{\kappa-1} \leq \delta_0$. We are now ready to apply Theorem~\ref{t-thm}. Recalling that $M \geq 2^{n/8}$, we want to choose $\epsilon$ such that 
\begin{align*}
\mbox{KL}(P_{i,T} | P_{j,T} )  \leq 16 T \mu^2 \left( 2 \tau n  \epsilon^2  \right)^{2(\kappa-1)}  \leq a \frac{n}{8} \log(2) \leq a \log M
\end{align*}
with an $a$ small enough so that we can apply the theorem. By setting $a=1/16$ and equating the two sides of the equation we have $\epsilon = \epsilon_T := \frac{1}{2 \sqrt{n}} \left(\frac{2}{\tau}\right)^{1/2}  \left( \frac{n \log(2)}{2048 \mu^2 T} \right)^{\frac{1}{4(\kappa-1)}}$ (note that this also implies a sequence of sets $\B_T$ by the definition of the functions in Lemma~\ref{baseFns}). Thus, the semi-distance satisfies
\begin{align*}
d(f_j,f_i) = ||x_j - x_i || \geq \sqrt{n/2} \epsilon_T \geq  \frac{1}{2 \sqrt{2}} \left(\frac{2}{\tau}\right)^{1/2}  \left( \frac{n \log(2)}{2048 \mu^2 T} \right)^{\frac{1}{4(\kappa-1)}} \ =: 2 s_T  \ .
\end{align*}
Applying Theorem~\ref{t-thm}  we  have
\begin{align*}
\inf_{\widehat{f}} \sup_{f\in\F}\P ( \|x_{\widehat f}-x_f\| \geq  s_T )   & \geq \inf_{\widehat{f}} \max_{i\in\{0,\dots,M\}} \P ( \|x_{\widehat f}-x_i\| \geq  s_T ) 
=\inf_{\widehat{f}} \max_{i\in\{0,\dots,M\}} \P ( d(\widehat{f},f_i) \geq  s_T ) \\
&\geq \ \textstyle\frac{\sqrt{M}}{1+\sqrt{M}} \left( 1- 2a - 2 \textstyle\sqrt{\frac{a}{\log M}} \right) > 1/7 \, ,
\end{align*}
where the final inequality holds since $M\geq 2$ and $a=1/16$.  Strong convexity implies that $f(x)-f(x_f) \geq \frac{\tau}{2}||{x}-x_f||^2$ for all $f\in \F$ and $x \in \B$. Therefore 
\begin{align*}
\inf_{\widehat{f}} \sup_{f\in\F} \P \left( f(x_{\widehat{f}})-f(x_f)  \geq \frac{\tau}{2}s_T^2\right)
&\geq   \inf_{\widehat{f}} \max_{i \in \{ 0, \dots, M\}}\P \left( f_i(x_{\widehat f}) - f_i(x_i) \geq  \frac{\tau}{2}s_T^2 \right)\\
&\geq \inf_{\widehat{f}} \max_{i\in\{0,\dots,M\}}\P \left( \frac{\tau}{2}\|x_{\widehat{f}}-x_i\|^2 \geq  \frac{\tau}{2}s_T^2 \right) \\
&= \inf_{\widehat{f}} \max_{i\in\{0,\dots,M\}}\P \left( \|x_{\widehat{f}}-x_i\| \geq  s_T \right) > 1/7 \, .
\end{align*}
Finally, applying Markov's inequality we have
\begin{align*}
\inf_{\widehat{f}} \sup_{f\in\F} \E \left[  f(x_{\widehat f}) - f(x_f) \right]  \geq \frac{1}{7} \left( \frac{1}{32}\right) \left( \frac{n \log(2)}{2048 \mu^2 T} .\right)^{\frac{1}{2(\kappa-1)}}\end{align*}

\subsection{Proof of Theorem \ref{comparisonThm} for $\kappa=1$} \label{kappaE1}
{To handle the case when $\kappa=1$ we use functions of the same form, but the construction is slightly different. Let $\ell$ be a positive integer and let $M=\ell^n$. Let $\{\xi_i\}_{i=1}^M$ be a set
of uniformly space points in $\B$ which we define to be the unit cube in $\R^n$, so that $\|\xi_i-\xi_j\|\geq\ell^{-1}$ for all $i\neq j$.  Define $f_i(x) := \frac{\tau}{2}||x- \xi_i ||^2$, $i=1,\dots,M$. Let $s:=\frac{1}{2\ell}$ so that $d(f_i,f_j) := ||x_i^* - x_j^*|| \geq 2 s$. 
Because $\kappa=1$, we have $\P\left( C_{f_i}(x,y) = \sign\{f_i(y)-f_i(x)\}\right)  \geq \mu$ for some $\mu>0$, all $i\in\{1,\dots,M\}$, and all $x,y\in \B$. We bound $\mbox{KL}(P_{i,T} || P_{j,T} )$ in exactly the same way as we bounded it in Section~\ref{kappaG1} except that now we have ${C_{f_i} (x_k,y_k)} \sim \text{Bernoulli}( \frac{1}{2}+\mu )$ and ${C_{f_j}(x_k,y_k)} \sim \text{Bernoulli}( \frac{1}{2}-\mu )$. It then follows that if we wish to apply the theorem, we want to choose $s$ so that
\begin{align*}
\mbox{KL}(P_{i,T} | P_{j,T} )  \leq 2 T \mu^2 / (1/2-\mu) \leq a\log M = a n \log\left( \textstyle\frac{1}{2 s}\right)
\end{align*}
for some $a < 1/8$. Using the same sequence of steps as in Section~\ref{kappaG1} we have
\begin{align*}
\inf_{\widehat{f}} \sup_{f\in\F} \E \left[ f(x_{\widehat f}) - f(x_f) \right]  \geq \frac{1}{7} \frac{\tau}{2} \left( \frac{1}{2} \right)^2 \exp\left\{ -\frac{128T \mu^2}{n(1/2-\mu)} \right\} .
\end{align*}
}

\subsection{Proof of Theorem \ref{evaluationThm}} \label{evaluationThmProof}
{Let $f_i$ for all $i=0,\dots,M$ be the functions considered in Lemma~\ref{baseFns}.
Recall that the evaluation oracle is defined to be $E_f(x):=f(x)+w$, where $w$ is a random variable (independent of all other random variables under consideration) with $\E[w]=0$ and $\E[w^2]=\sigma^2>0$. 
Let $\{x_k\}_{k=1}^n$ be a sequence of points in $\B \subset \R^n$ and let $\{E_f(x_k)\}_{k=1}^T$ denote the corresponding sequence of noisy evaluations of $f\in\F$. For $\ell=1,\dots,T$ let $P_{i,\ell}$ denote the joint probability distribution of $\{x_k,E_{f_i}(x_k)\}_{k=1}^\ell$, let $Q_{i,\ell}$ denote the conditional distribution of $E_{f_i}(x_k)$ given $x_k$, and let $S_\ell$ denote the conditional distribution of $x_\ell$ given $\{x_k,E_f(x_k)\}_{k=1}^{\ell-1}$.  $S_\ell$ is a function of the underlying optimization algorithm and does not depend on $i$.
We can now bound the KL divergence between any two hypotheses as in Section~\ref{kappaG1}:
\begin{align*}
\mbox{KL}(P_{i,T} || P_{j,T} ) &\leq T \sup_{x_1\in \B} \E_{P_{i,1}} \left[ \E_{P_{i,1}} \left[ \log \frac{ Q_{i,1} }{ Q_{j,1} } \bigg | x_1\right]\right] \ .
\end{align*}
To compute a bound, let us assume that $w$ is Gaussian distributed.  Then
\begin{align*}
\mbox{KL}(P_{i,T} || P_{j,T} ) &\leq T \sup_{z\in\B} \mbox{KL}\left( \mathcal{N}(f_i(z),\sigma^2) ||\mathcal{N}(f_j(z),\sigma^2) \right) \\
&= \frac{T}{2\sigma^2}   \sup_{z\in\B} |f_i(z)-f_j(z)|^2 \leq \frac{T}{2\sigma^2}   \left( 2 \tau n  \epsilon^2 \right)^{2}
\end{align*}
by the third claim of Lemma~\ref{baseFns}.  We then repeat the same procedure as in Section~\ref{kappaG1} to attain 
\begin{align*}
\inf_{\widehat{f}} \sup_{f\in\F} \E \left[  f(x_{\widehat f}) - f(x_f) \right]  \geq \frac{1}{7} \left( \frac{1}{32}\right)  \left( \frac{n \sigma^2 \log(2)}{64 T} \right)^{\frac{1}{2}}.
\end{align*}}

\section{Upper bounds} \label{upper}
The algorithm that achieves the upper bound using a pairwise comparison oracle  is a combination of standard techniques and methods from the convex optimization and statistical learning literature. The algorithm is explained in full detail in Appendix~\ref{suppMaterials}, and is summarized as follows. At each iteration the algorithm picks a coordinate uniformly at random from the $n$ possible dimensions and  then performs an approximate line search. By exploiting the fact that the function is strongly convex  with  Lipschitz gradients, one guarantees using standard arguments that the approximate line search makes a sufficient decrease in the objective function value in expectation \cite[Ch.9.3]{boyd}.  If the pairwise comparison oracle made no errors then the approximate line search is accomplished by a binary-search-like scheme, essentially a golden section line-search algorithm \cite{brent}. { However, when responses from the oracle are only probably correct we make the line-search robust to errors by repeating the same query until we can be confident about the true, uncorrupted direction of the pairwise comparison using a standard procedure from the active learning literature \cite{matti} (a similar technique was also implemented for the bandit setting of derivate-free optimization \cite{agarwal2011stochastic}).} Because the analysis of each component is either known or elementary, we only sketch the proof here and leave the details to the supplementary materials. 

\subsection{Coordinate descent}

Given a candidate solution $x_k$ after $k\geq0$ iterations, the algorithm defines a search direction $d_k = \mathbf{e}_i$ where $i$ is chosen uniformly at random from the possible $n$ dimensions and $\mathbf{e}_i$ is a vector of all zeros except for a one in the $i$th coordinate. We note that while we only analyze the case where the search direction $d_k$ is a coordinate direction, an analysis with the same result can be obtained with $d_k$ chosen uniformly from the unit sphere.  Given $d_k$, a line search is then performed to find an $\alpha_k \in \R$ such that $f(x_{k+1}) - f(x_k)$ is sufficiently small where $x_{k+1} = x_k + \alpha_k d_k$. In fact, as we will see in the next section, for some input parameter $\eta >0$, the line search is guaranteed to return an $\alpha_k$ such that $|\alpha_k-\alpha^*| \leq \eta$ where $\alpha_* = \min_{\alpha \in \R} f(x_k+d_k \alpha^*)$. Using the fact that the gradients of $f$ are Lipschitz $(L)$ we have
\begin{align*}
 f( x_k + \alpha_k d_k) -  f( x_k + \alpha^* d_k) \leq \frac{L}{2} || (\alpha_k - \alpha^*) d_k||^2 =   \frac{L}{2} | \alpha_k - \alpha^*|^2 \leq \frac{L}{2} \eta^2.
\end{align*}
If we define $\hat{\alpha_k} = -  \frac{ \langle \nabla f(x_k), d_k \rangle }{L} $ then we have
\begin{align*}
f( x_k + \alpha_k d_k) - f( x_k) &\leq f( x_k + \alpha^* d_k)  - f( x_k) + \frac{L}{2} \eta^2\\
&\leq     f( x_k + \hat{\alpha}_k d_k)  - f( x_k) + \frac{L}{2} \eta^2 \leq -   \frac{ \langle \nabla f(x_k), d_k \rangle^2 }{2L} + \frac{L}{2} \eta^2
\end{align*} 
where the last line follows from applying the fact that the gradients are Lipschitz $(L)$. Arranging the bound and taking the expectation with respect to $d_k$ we get
\begin{align*}
\E\left[ f(x_{k+1})-f(x^*) \right] -\textstyle\frac{L}{2} \eta^2 &\leq \E \left[ f(x_k)-f(x^*) \right]-  \textstyle\frac{\E \left[ || \nabla f(x_k)||^2\right]  }{2nL} \leq  \E \left[ f(x_k)-f(x^*) \right] \left( 1- \frac{ \tau}{4nL} \right) 
\end{align*}
where the second inequality follows from the fact that $f$ is strongly convex $(\tau)$. If we define $\rho_k := \E \left[ f(x_k) - f(x^*) \right]$ then we equivalently have
\begin{align*}
\rho_{k+1} -  \frac{2nL^2 \eta^2}{\tau} \leq \left( 1- \frac{\tau}{4nL} \right) \left(\rho_k -  \frac{2nL^2 \eta^2}{\tau} \right)  \leq \left( 1- \frac{\tau}{4nL} \right)^k \left(\rho_0 - \frac{2nL^2 \eta^2}{\tau}\right)
\end{align*}
which leads to the following result.

\begin{theorem} \label{coordDescentThm}
Let $f\in\F_{\tau,L,\B}$ with $\B = \R^n$. For any $\eta>0$ assume the line search returns an $\alpha_k$ that is within $\eta$ of the optimal after at most $T_\ell(\eta)$ queries from the pairwise comparison oracle. If $x_K$ is an estimate of $x^* = \arg \min_x f(x)$ after requesting no more than $K$ pairwise comparisons, then    
\begin{align*}
\sup_f \E[ f(x_{K})-f(x_*) ] \leq \frac{4nL^2 \eta^2}{\tau}  \hspace{.3in} \text{ whenever } \hspace{.3in} K \geq \frac{4n L}{ \tau} \log\left(  \frac{f(x_0)-f(x^*)}{\eta^2 2 n L^2 / \tau}   \right) T_\ell(\eta)
\end{align*}
where  the expectation is with respect to the random choice of $d_k$ at each iteration.
\end{theorem}

This implies that if we wish $\sup_f \E[ f(x_{K})-f(x_*) ]  \leq \epsilon$ it suffices to take $\eta = \sqrt{\frac{\epsilon \tau}{4nL^2}}$ so that at most $\frac{4n L}{ \tau} \log\left(  \frac{f(x_0)-f(x^*)}{\epsilon/2}   \right) T_\ell\left(\sqrt{\frac{\epsilon \tau}{4nL^2}}\right)$ pairwise comparisons are requested. 

\subsection{Line search}
This section is concerned with minimizing a function $f(x_k + \alpha_k d_k)$ over some $\alpha_k \in \R$. In particular, we wish to find an $\alpha_k \in \R$ such that $|\alpha_k-\alpha^*| \leq \eta$ where $\alpha_* = \min_{\alpha \in \R} f(x_k+d_k \alpha^*)$. First assume that the function comparison oracle makes no errors. The line search operates by maintaining a pair of boundary points $\alpha^+$, $\alpha^-$ such that if at some iterate we have $\alpha^* \in [\alpha^-, \alpha^+]$ then at the next iterate, we are guaranteed that $\alpha^*$ is still contained inside the boundary points but $|\alpha^+ -\alpha^-| \leftarrow \frac{1}{2} |\alpha^+ -\alpha^-|$. An initial set of boundary points $\alpha^+>0$ and $\alpha^-<0$ are found using simple binary search. Thus, regardless of how far away or close  $\alpha^*$ is, we converge to it exponentially fast. Exploiting the fact that $f$ is strongly convex $(\tau)$ with Lipschitz $(L)$ gradients we can bound how far away or close $\alpha^*$ is from our initial iterate. 

\begin{theorem} \label{noiselessThm}
Let $f\in\F_{\tau,L,\B}$ with $\B = \R^n$ and let $C_f$ be a function comparison oracle that makes no errors. Let $x \in \R^n$ be an initial position and let $d \in \R^n$ be a search direction with $||d||=1$. If $\alpha_K$ is an estimate of $\alpha^* = \arg \min_\alpha f(x+d \alpha)$ that is output from the line search after requesting no more than  $K$ pairwise comparisons, then for any $\eta >0$ 
\begin{align*}
|\alpha_K - \alpha^*| \leq \eta \hspace{.3in} \text{ whenever } \hspace{.3in} K \geq 2\log_2\left( \frac{ 256 L \left( f(x)- f(x+ d \, \alpha^*) \right)}{\tau^2 \eta^2} \right).\end{align*}
\end{theorem}

\subsection{Making the line search robust to errors}
Now assume that the responses from the pairwise comparison oracle are only probably correct in accordance with the model introduced above. Essentially, the robust procedure runs the line search as if the oracle made no errors except that each time a comparison is needed,  the oracle is repeatedly queried until we can be confident about the true direction of the comparison. This strategy applied to active learning is well known because of its simplicity and its ability to adapt to unknown noise conditions \cite{matti}. However, we mention that when used in this way, this sampling procedure is known to be sub-optimal so in practice, one may want to implement a more efficient approach like that of \cite{castroNowak}. Nevertheless, we have the following lemma.
\begin{lemma} \cite{matti} \label{lemmaRepeat}
For any $x,y \in \B$ with $\P \left( C_f(x,y)= \sign\{f(y)-f(x)\} \right) = p$,  with probability at least $1-\delta$ the coin-tossing algorithm of \cite{matti} correctly identifies the sign of $\E\left[ C_f(x,y) \right]$ and requests no more than $\frac{\log(2 /\delta)}{4|1/2-p|^2} \log_2 \left( \frac{ \log(2 /\delta)}{4|1/2-p|^2}  \right) $ pairwise comparisons.
 \end{lemma}
It would be convenient if we could simply apply the result of Lemma~\ref{lemmaRepeat} to our line search procedure. Unfortunately, if we do this there is no guarantee that $|f(y)- f(x)|$ is bounded below so for the case when $\kappa >1$,  it would be impossible to lower bound $|1/2-p|$ in the lemma. To account for this, we will sample at multiple locations per iteration as opposed to just two in the noiseless algorithm to ensure that we can always lower bound $|1/2-p|$. Intuitively, strong convexity ensures that $f$ cannot be arbitrarily flat so for any three equally spaced points $x,y,z$ on the line $d_k$, if $f(x)$ is equal to $f(y)$, then it follows that the absolute difference between $f(x)$ and $f(z)$ must be bounded away from zero. Applying this idea and union bounding over the total number of times one must call the coin-tossing algorithm, one finds that with probability at least $1-\delta$, the total number of calls to the pairwise comparison oracle over the course of the whole algorithm does not exceed $\widetilde{O}\left( \frac{ nL}{\tau}  \left(\frac{n}{\epsilon}\right)^{2(\kappa-1)} \log^2\left(  \frac{ f(x_0)-f(x^*)  }{\epsilon}   \right)  \log(n/\delta)  \right).$ By finding a $T>0$ that satisfies this bound for any $\epsilon$ we see that this is equivalent to a rate of ${O} \left( n \log(n/\delta) \left( \frac{n}{T} \right)^{\frac{1}{2(\kappa-1)}} \right)$ for $\kappa>1$ and ${O} \left( \exp\left\{{-c\sqrt{\frac{T}{n \log(n/\delta)} }} \right\} \right)$ for $\kappa=1$, ignoring $\text{polylog}$ factors.

{ 
\section{Conclusion} \label{conclusion}
This paper presented lower bounds on the performance of derivative-free optimization for (i) an oracle that provides noisy function evaluations and (ii) an oracle that provides probably correct boolean comparisons between function evaluations. Our results were proven for the class of strongly convex functions but because this class is a subset of all, possibly non-convex functions, our lower bounds hold for much larger classes as well. Under both oracle models we showed that the expected error decays like $\Omega\left( (n/T)^{1/2} \right)$. Furthermore, for the class of strongly convex functions with Lipschitz gradients, we proposed an algorithm that achieves a rate of  $\widetilde{O}\left(n (n/T)^{1/2} \right)$ for both oracle models which shows that the lower bounds are tight with respect to the dependence on the number of iterations $T$ and no more than a factor of $n$ off in terms of the dimension.

A number of open questions still remain. In particular, one would like to resolve the gap between the lower  and upper bounds with respect to the dependence on the dimension. Due to real world constraints, it is also desirable to extend the pairwise comparison algorithm to operate under the conditions of constrained optimization where $\B$ is a convex, proper subset of $\R^d$. Also, while the analysis of our algorithm relies heavily on the assumption that the function is strongly convex with Lipschitz gradients, it is unclear whether these assumptions are necessary to achieve the same rates of convergence. Developing a practical algorithm that achieves our lower bounds and does not suffer from these limiting assumptions would be a significant contribution.

  }
  
   \newpage 
 
  \bibliography{optAlgorithm2_SConly}

\begin{thebibliography}{10}

\bibitem{eitrich}
T.~Eitrich and B.~Lang.
\newblock Efficient optimization of support vector machine learning parameters
  for unbalanced datasets.
\newblock {\em Journal of computational and applied mathematics},
  196(2):425--436, 2006.

\bibitem{oeuvray}
R.~Oeuvray and M.~Bierlaire.
\newblock A new derivative-free algorithm for the medical image registration
  problem.
\newblock {\em International Journal of Modelling and Simulation},
  27(2):115--124, 2007.

\bibitem{conn}
A.R. Conn, K.~Scheinberg, and L.N. Vicente.
\newblock {\em Introduction to derivative-free optimization}, volume~8.
\newblock Society for Industrial Mathematics, 2009.

\bibitem{PowellLearning}
Warren~B. Powell and Ilya~O. Ryzhov.
\newblock {\em Optimal Learning}.
\newblock John Wiley and Sons, 2012.

\bibitem{nesterov}
Y.~Nesterov.
\newblock Random gradient-free minimization of convex functions.
\newblock {\em CORE Discussion Papers}, 2011.

\bibitem{srinivas}
N.~Srinivas, A.~Krause, S.M. Kakade, and M.~Seeger.
\newblock Gaussian process optimization in the bandit setting: No regret and
  experimental design.
\newblock {\em Arxiv preprint arXiv:0912.3995}, 2009.

\bibitem{storn}
R.~Storn and K.~Price.
\newblock Differential evolution--a simple and efficient heuristic for global
  optimization over continuous spaces.
\newblock {\em Journal of global optimization}, 11(4):341--359, 1997.

\bibitem{agarwal2011stochastic}
A.~Agarwal, D.P. Foster, D.~Hsu, S.M. Kakade, and A.~Rakhlin.
\newblock Stochastic convex optimization with bandit feedback.
\newblock {\em Arxiv preprint arXiv:1107.1744}, 2011.

\bibitem{nemirovski}
A.~Nemirovski, A.~Juditsky, G.~Lan, and A.~Shapiro.
\newblock Robust stochastic approximation approach to stochastic programming.
\newblock {\em SIAM Journal on Optimization}, 19(4):1574, 2009.

\bibitem{protasov}
V.~Protasov.
\newblock Algorithms for approximate calculation of the minimum of a convex
  function from its values.
\newblock {\em Mathematical Notes}, 59:69--74, 1996.
\newblock 10.1007/BF02312467.

\bibitem{raginskyConvexComplexity}
M.~Raginsky and A.~Rakhlin.
\newblock Information-based complexity, feedback, and dynamics in convex
  programming.
\newblock {\em Information Theory, IEEE Transactions on}, (99):1--1, 2011.

\bibitem{thurstone}
L.L. Thurstone.
\newblock A law of comparative judgment.
\newblock {\em Psychological Review; Psychological Review}, 34(4):273, 1927.

\bibitem{yueDuelingBandits}
Y.~Yue, J.~Broder, R.~Kleinberg, and T.~Joachims.
\newblock The k-armed dueling bandits problem.
\newblock {\em Journal of Computer and System Sciences}, 2012.

\bibitem{activeRanking}
K.G. Jamieson and R.D. Nowak.
\newblock Active ranking using pairwise comparisons.
\newblock {\em Arxiv preprint arXiv:1109.3701}, 2011.

\bibitem{nemirovskyYudin}
A.S. Nemirovsky and D.B. Yudin.
\newblock Problem complexity and method efficiency in optimization.
\newblock 1983.

\bibitem{agarwalBanditDFO}
A.~Agarwal, D.P. Foster, D.~Hsu, S.M. Kakade, and A.~Rakhlin.
\newblock Stochastic convex optimization with bandit feedback.
\newblock {\em Arxiv preprint arXiv:1107.1744}, 2011.

\bibitem{alekh}
A.~Agarwal, P.L. Bartlett, P.~Ravikumar, and M.J. Wainwright.
\newblock Information-theoretic lower bounds on the oracle complexity of
  stochastic convex optimization.
\newblock {\em Information Theory, IEEE Transactions on}, (99):1--1, 2010.

\bibitem{agarwalMultiPointBandit}
A.~Agarwal, O.~Dekel, and L.~Xiao.
\newblock Optimal algorithms for online convex optimization with multi-point
  bandit feedback.
\newblock In {\em Conference on Learning Theory (COLT)}, 2010.

\bibitem{ghadimi}
S.~Ghadimi and G.~Lan.
\newblock Stochastic first-and zeroth-order methods for nonconvex stochastic
  programming.
\newblock 2012.

\bibitem{tsybakov}
A.B. Tsybakov.
\newblock {\em Introduction to nonparametric estimation}.
\newblock Springer Verlag, 2009.

\bibitem{castroNowak}
R.M. Castro and R.D. Nowak.
\newblock Minimax bounds for active learning.
\newblock {\em Information Theory, IEEE Transactions on}, 54(5):2339--2353,
  2008.

\bibitem{supp}
Anonymous.
\newblock Supplementary material.
\newblock {\em Advances in Neural Information Processing Systems (NIPS)}, 2012.

\bibitem{boyd}
S.P. Boyd and L.~Vandenberghe.
\newblock {\em Convex optimization}.
\newblock Cambridge Univ Pr, 2004.

\bibitem{brent}
R.P. Brent.
\newblock {\em Algorithms for minimization without derivatives}.
\newblock Dover Pubns, 2002.

\bibitem{matti}
M.~K{\"a}{\"a}ri{\"a}inen.
\newblock Active learning in the non-realizable case.
\newblock In {\em Algorithmic Learning Theory}, pages 63--77. Springer, 2006.

\end{thebibliography}

  \newpage
  
\appendix

\section{Bounds on $(\kappa,\mu,\delta_0)$ for some distributions} \label{distributionParams}
In this section we relate the function evaluation oracle to the function comparison oracle for some common distributions. That is, if $E_f(x) = f(x) + w$ for some random variable $w$, we lower bound the probability $\eta(y,x) := \P( \text{sign}\{E_f(y)-E_f(x) \} = \text{sign}\{f(y)-f(x)\} )$ in terms of the parameterization of $(\ref{funcCompModel})$.
\begin{lemma}
Let $w$ be a Gaussian random variable with mean zero and variance $\sigma^2$. Then ${\eta(y,x) \geq \frac{1}{2} + \min\left\{  \frac{1}{\sqrt{2 \pi e}} , \frac{1}{\sqrt{4\pi \sigma^2 e}} |f(y)-f(x)|\right\}}$.
\end{lemma}
\begin{proof}
Notice that $ \eta(y,x)=\P( Z + | f(y)-f(x) | / \sqrt{2 \sigma^2} \geq 0 ) $ where $Z$ is a standard normal. The result follows by lower bounding the density of $Z$ by $ \frac{1}{\sqrt{2 \pi e}} \1\{ |Z| \leq 1 \}$ and integrating where $\1\{ \cdot \}$ is equal to one when its arguments are true and zero otherwise.
\end{proof}

We say $w$ is a 2-sided gamma distributed random variable if its density is given by $\frac{\beta^\alpha}{2 \Gamma(\alpha)} |x|^{\alpha-1} e^{-\beta |x|}$ for $x \in [- \infty, \infty]$ and $\alpha, \beta>0$. Note that this distribution is unimodal only for $\alpha \in (0,1]$ and is equal to a Laplace distribution for $\alpha = 1$. This distribution has variance $\sigma^2 = \alpha / \beta^2$.
 \begin{lemma}
Let $w$ be a 2-sided gamma distributed random variable with parameters $\alpha \in (0,1]$ and $\beta >0$. Then ${\eta(y,x) \geq \frac{1}{2} + \min\left\{ \frac{1}{4 \alpha^2 \Gamma(\alpha)^2} \left(\frac{\alpha}{e} \right)^{2\alpha} , \frac{(\beta/2e)^{2\alpha}}{4 \alpha^2 \Gamma(\alpha)^2} | f(y)-f(x) |^{2\alpha} \right\}}$.
\end{lemma}
\begin{proof} Let $E_f(y)=f(y)+w$ and $E_f(x) = f(x)+w'$ where $w$ and $w'$ are i.i.d. 2-sided gamma distributed random variables. If we lower bound $ e^{-\beta |x|}$ with $e^{-\alpha} \1\{ |x| \leq \alpha/\beta \}$ and integrate we find that $\P( -t/2 \leq w \leq 0) \geq  \min\left\{ \frac{1}{2 \alpha \Gamma(\alpha)} \left(\frac{\alpha}{e} \right)^\alpha , \frac{(\beta/e)^\alpha}{2 \alpha \Gamma(\alpha)} (t/2)^\alpha \right\}$. And by the symmetry and independence of $w$ and $w'$ we have ${\P(-t \leq w - w'  ) \geq \frac{1}{2} + \P( -t/2 \leq w \leq 0) \P( -t/2 \leq w \leq 0)}$. \end{proof}
While the bound in the lemma immediately above can be shown to be loose, these two lemmas are sufficient to show that the entire range of $\kappa \in (1,2]$ is possible. 

\section{Upper Bounds - Extended} \label{suppMaterials}

The algorithm that achieves the upper bound using a pairwise comparison oracle  is a combination of a few standard techniques and methods pulled from the convex optimization and statistical learning literature. The algorithm can be summarized as follows. At each iteration the algorithm picks a coordinate uniformly at random from the $n$ possible dimensions and  then performs an approximate line search. By exploiting the fact that the function is strongly convex  with  Lipschitz gradients, one guarantees using standard arguments that the approximate line search makes a sufficient decrease in the objective function value in expectation \cite[Ch.9.3]{boyd}.  If the pairwise comparison oracle made no errors then the approximate line search is accomplished by a binary-search-like scheme that is known in the literature as the golden section line-search algorithm \cite{brent}. However, when responses from the oracle are only probably correct we make the line-search robust to errors by repeating the same query until we can be confident about the true, uncorrupted direction of the pairwise comparison using a standard procedure from the active learning literature \cite{matti}.

\subsection{Coordinate descent algorithm}
\begin{figure}[h]
\centering
\fbox{\parbox[b]{3.9in}{{\underline{\bf $n$-dimensional Pairwise comparison algorithm}}    \\
Input: $x_0 \in \R^n$,  $\eta \geq0$\\
\textbf{For} k=0,1,2,\dots\\
\text{\ \ \ \ } Choose $d_k=\mathbf{e}_i$ for $i\in\{1,\dots,n\}$ chosen uniformly at random\\
\text{\ \ \ \ } Obtain $\alpha_k$ from a line-search such that \\
\text{\ \ \ \ \ \ \ \ } $|\alpha_k-\alpha^*|\leq\eta$ where $\alpha^*=  \arg \min_\alpha f( x_k + \alpha d_k) $\\
\text{\ \ \ \ } $x_{k+1} = x_k + \alpha_k d_k$\\
\textbf{end}
}}
\caption{Algorithm to minimize a convex function in $d$ dimensions. Here $\mathbf{e}_i$ is understood to be a vector of all zeros with a one in the $i$th position.}
\label{fig:alg4d}
\end{figure} 

\begin{theorem} \label{coordDescentThm}
Let $f\in\F_{\tau,L,\B}$ with $\B = \R^n$. For any $\eta>0$ assume the line search in the algorithm of Figure~\ref{fig:alg4d} requires at most $T_\ell(\eta)$ queries from the pairwise comparison oracle.  If $x_K$ is an estimate of $x^* = \arg \min_x f(x)$ after requesting no more than $K$ pairwise comparisons, then    
\begin{align*}
\sup_f \E[ f(x_{K})-f(x_*) ] \leq \frac{4nL^2 \eta^2}{\tau}  \hspace{.3in} \text{ whenever } \hspace{.3in} K \geq \frac{4n L}{ \tau} \log\left(  \frac{f(x_0)-f(x^*)}{\eta^2 2 n L^2 / \tau}   \right) T_\ell(\eta)
\end{align*}
where  the expectation is with respect to the random choice of $d_k$ at each iteration.
\end{theorem}
\begin{proof}
First note that $||d_k|| =1$ for all $k$ with probability $1$.  Because the gradients of $f$ are Lipschitz $(L)$ we have from Taylor's theorem
\begin{align*}
f(x_{k+1}) \leq f(x_k) + \langle \nabla f(x_k), \alpha_k d_k \rangle + \frac{\alpha_k^{2} L}{2}.
\end{align*}
Note that the right-hand-side is convex in $\alpha_k$ and is minimized by 
\begin{align*}
\hat{\alpha_k} = -  \frac{ \langle \nabla f(x_k), d_k \rangle }{L} .
\end{align*}
However, recalling how $\alpha_k$ is chosen, if $\alpha^*=  \arg \min_\alpha f( x_k + \alpha d_k) $ then we have
\begin{align*}
 f( x_k + \alpha_k d_k) -  f( x_k + \alpha^* d_k) \leq \frac{L}{2} || (\alpha_k - \alpha^*) d_k||^2 =   \frac{L}{2} | \alpha_k - \alpha^*|^2 \leq \frac{L}{2} \eta^2.
\end{align*}
This implies 
\begin{align*}
f( x_k + \alpha_k d_k) - f( x_k) &\leq f( x_k + \alpha^* d_k)  - f( x_k) + \frac{L}{2} \eta^2\\
&\leq     f( x_k + \hat{\alpha}_k d_k)  - f( x_k) + \frac{L}{2} \eta^2\\
&\leq -   \frac{ \langle \nabla f(x_k), d_k \rangle^2 }{2L} + \frac{L}{2} \eta^2.
\end{align*}
Taking the expectation with respect to $d_k$, we have
\begin{align*}
\E \left[ f(x_{k+1}) \right] &\leq \E \left[ f(x_k) \right]  -  \E \left[ \frac{ \langle \nabla f(x_k), d_k \rangle^2 }{2L} \right] + \frac{L}{2} \eta^2\\
&= \E \left[ f(x_k) \right]  - \E \left[  \E \left[ \frac{ \langle \nabla f(x_k), d_k \rangle^2 }{2L}  \bigg| d_0,\dots,d_{k-1} \right] \right] + \frac{L}{2} \eta^2\\
&= \E \left[ f(x_k) \right]  - \E \left[ \frac{ || \nabla f(x_k)||^2 }{2nL} \right]  + \frac{L}{2} \eta^2
\end{align*}
where we applied the law of iterated expectation. Let $x^* = \arg \min_x f(x)$ and note that $x^*$ is a unique minimizer by strong convexity $(\tau)$. Using the previous calculation we have
\begin{align*}
\E\left[ f(x_{k+1})-f(x^*) \right] - \textstyle\frac{L}{2} \eta^2 &\leq \E \left[ f(x_k)-f(x^*) \right]-  \textstyle\frac{\E \left[ || \nabla f(x_k)||^2\right]  }{2nL}  \leq  \E \left[ f(x_k)-f(x^*) \right] \left( 1- \frac{ \tau}{4nL} \right)
\end{align*}
where the second inequality follows from
\begin{align*}
\left( f(x_k)-f(x^*) \right)^2 \leq& \left( \langle \nabla f(x_k), x_k-x^* \rangle \right)^2\\ 
\leq& || \nabla f(x_k) ||^2 || x_k-x^* ||^2 \leq || \nabla f(x_k) ||^2 \left(\frac{\tau}{2}\right)^{-1}\left( f(x_k)-f(x^*) \right).
\end{align*}
If we define $\rho_k := \E \left[ f(x_k) - f(x^*) \right]$ then we equivalently have
\begin{align*}
\rho_{k+1} -  \frac{2nL^2 \eta^2}{\tau} \leq \left( 1- \frac{\tau}{4nL} \right) \left(\rho_k -  \frac{2nL^2 \eta^2}{\tau} \right)  \leq \left( 1- \frac{\tau}{4nL} \right)^k \left(\rho_0 - \frac{2nL^2 \eta^2}{\tau}\right)
\end{align*}
which completes the proof.
\end{proof}
This implies that if we wish $\sup_f \E[ f(x_{K})-f(x_*) ]  \leq \epsilon$ it suffices to take $\eta = \sqrt{\frac{\epsilon \tau}{4nL^2}}$ so that at most $\frac{4n L}{ \tau} \log\left(  \frac{f(x_0)-f(x^*)}{\epsilon/2}   \right) T_\ell\left(\sqrt{\frac{\epsilon \tau}{4nL^2}}\right)$ pairwise comparisons are requested. 

\subsection{Line search}
This section is concerned with minimizing a function $f(x_k + \alpha d_k)$ over some $\alpha \in \R$. Because we are minimizing over a single variable, $\alpha$, we will restart the indexing at $0$ such that the line search algorithm produces a sequence $\alpha_0,\alpha_1,\dots,\alpha_{K'}$. This indexing should not be confused with the indexing of the iterates $x_1,x_2,\dots,x_K$. We will first present an algorithm that assumes the pairwise comparison oracle makes no errors and then extend the algorithm to account for the noise model introduced in Section~\ref{problemForm}.

Consider the algorithm of Figure~\ref{fig:alg1d}. At each iteration, one is guaranteed to eliminate at least $1/2$  the search space at each iteration such that at least $1/4$ the search space is discarded for every pairwise comparison that is requested. However, with a slight modification to the algorithm, one can guarantee a greater fraction of removal (see the golden section line-search algorithm). We use this sub-optimal version for simplicity because it will help provide intuition for how the robust version of the algorithm works.

\begin{figure}[h]
\centering
\fbox{\parbox[b]{3.7in}{{\underline{\bf One Dimensional Pairwise comparison algorithm}}    \\
Input: $x \in \R^n$, $d \in \R^n$, $\eta>0$\\
Initialize: $\alpha_0=0$, $\alpha_0^+= \alpha_0+1$, $\alpha_0^- = \alpha_0 - 1$, $k=0$\\
\textbf{If} \, $C_f(x, x+ d\, \alpha_0^+)>0$ and $C_f(x, x+ d\, \alpha_0^-)<0$\\
\text{ \ \ \ \ } $\alpha_0^+ = 0$\\
\textbf{end}\\
\textbf{If} \, $C_f(x, x+ d\, \alpha_0^-)>0$ and $C_f(x,x+ d\,  \alpha_0^+)<0$\\
\text{ \ \ \ \ } $\alpha_0^- = 0$\\
\textbf{end}\\
\textbf{While} \, $C_f(x, x+ d \, \alpha_k^+)<0$\\
\text{ \ \ \ \ } $\alpha_{k+1}^+ = 2\alpha_k^+$, $k=k+1$\\
\textbf{end}\\
\textbf{While} \, $C_f(x,x+ d \,  \alpha_k^-)<0$\\
\text{ \ \ \ \ } $\alpha_{k+1}^- = 2\alpha_k^- $, $k=k+1$\\
\textbf{end}\\
$\alpha_k = \frac{1}{2} (\alpha_k^-+ \alpha_k^+)$\\
\textbf{While} $|\alpha_k^+-\alpha_k^-| \geq \eta/2$\\
\text{\ \ \ \ } \textbf{if} \, $C_f(x+ d \, \alpha_k,x+ d \, \frac{1}{2} (\alpha_k+ \alpha_k^+))<0$\\
\text{\ \ \ \ \ \ } $\alpha_{k+1} = \frac{1}{2} (\alpha_k+ \alpha_k^+)$, $\alpha_{k+1}^+ = \alpha_k^+$, $\alpha_{k+1}^- = \alpha_k$\\
\text{\ \ \ \ } \textbf{else if} \, $C_f(x+ d \, \alpha_k,x+ d \, \frac{1}{2} (\alpha_k+ \alpha_k^-))<0$\\
\text{\ \ \ \ \ \ } $\alpha_{k+1} = \frac{1}{2} (\alpha_k+ \alpha_k^-)$, $\alpha_{k+1}^+ = \alpha_k$, $\alpha_{k+1}^- = \alpha_k^-$\\
\text{\ \ \ \ } \textbf{else} \\
\text{\ \ \ \ \ \ } $\alpha_{k+1} = \alpha_k$, $\alpha_{k+1}^+ = \frac{1}{2} (\alpha_k+ \alpha_k^+)$, $\alpha_{k+1}^- = \frac{1}{2} (\alpha_k+ \alpha_k^-)$\\
\text{\ \ \ \ } \textbf{end} \\
\textbf{end}\\
Output: $\alpha_k$
}}
\caption{Algorithm to minimize a convex function in one dimension.}
\label{fig:alg1d}
\end{figure} 

\begin{theorem} \label{noiselessThm}
Let $f\in\F_{\tau,L,\B}$ with $\B = \R^n$ and let $C_f$ be a function comparison oracle that makes no errors. Let $x \in \R^n$ be an initial position and let $d \in \R^n$ be a search direction with $||d||=1$. If $\alpha_K$ is an estimate of $\alpha^* = \arg \min_\alpha f(x+d \alpha)$ that is output from the  algorithm of Figure~\ref{fig:alg1d} after requesting no more than  $K$ pairwise comparisons, then for any $\eta >0$ 
\begin{align*}
|\alpha_K - \alpha^*| \leq \eta \hspace{.3in} \text{ whenever } \hspace{.3in} K \geq 2\log_2\left( \frac{ 256 L \left( f(x)- f(x+ d \, \alpha^*) \right)}{\tau^2 \eta^2} \right).\end{align*}
\end{theorem}
\begin{proof}
First note that if $\alpha_K$ is output from the algorithm, we have $\frac{1}{2} |\alpha_K-\alpha^*| \leq  |\alpha_K^+-\alpha_K^-| \leq \frac{1}{2} \eta$, as desired.  

We will handle the cases when $|\alpha^*|$ is greater than one and less than one separately. First assume that $|\alpha^*| \geq 1$. Using the fact that $f$ is strongly convex $(\tau)$, it is straightforward to show that immediately after exiting the initial while loops, $(i)$  at most $2+ \frac{1}{2}\log_2\left( \frac{8}{\tau} \left( f(x)- f(x+ d \, \alpha^*) \right) \right)$ pairwise comparisons were requested, $(ii)$  $\alpha_* \in [\alpha_k^-, \, \alpha_k^+]$, and $(iii)$ $|\alpha_k^+-\alpha_k^-| \leq \left( \frac{8}{\tau} \left( f(x)- f(x+ d\, \alpha^*) \right) \right)^{1/2}$. We also have that $\alpha_* \in [\alpha_{k+1}^-, \, \alpha_{k+1}^+]$ if $\alpha_* \in [\alpha_k^-, \, \alpha_k^+]$ for all $k$.  Thus, it follows that
\begin{align*}
|\alpha_{k+l}^+-\alpha_{k+l}^-| =  2^{-l} |\alpha_k^+-\alpha_k^-|  \leq 2^{-l} \left( \frac{8}{\tau} \left( f(x)- f(x+ d\, \alpha^*) \right) \right)^{1/2}.
\end{align*} 
To make the right-hand-side less than or equal to $\eta/2$, set $l = \log_2\left( \frac{\left( \frac{8}{\tau} \left(  f(x)- f(x+ d\, \alpha^*) \right) \right)^{1/2}}{\eta/2} \right)$. This brings the total number of pairwise comparison requests to no more than $2\log_2\left( \frac{ 32 \left(  f(x)- f(x+ d\, \alpha^*) \right)}{\tau \eta} \right)$.

Now  assume that $|\alpha^*| \leq 1$. A straightforward calculation shows that the while loops will terminate after requesting at most $2+ \frac{1}{2} \log_2 \left( \frac{L}{\tau} \right)$ pairwise comparisons. And immediately after exiting the while loops we have $|\alpha_k^+-\alpha_k^-| \leq 2$. It follows by the same arguments of above that if we want $|\alpha_{k+l}^+-\alpha_{k+l}^-|  \leq \eta/2$ it suffices to set $l=\log_2\left( \frac{4}{\eta} \right)$. This brings the total number of pairwise comparison requests to no more than $2 \log_2\left( \frac{8L}{\tau \eta} \right) $. For sufficiently small $\eta$ both cases are positive and the result follows from adding the two.
\end{proof}

This implies that if the function comparison oracle makes no errors and it is given an iterate $x_k$ and direction $d_k$ then $T_\ell\left(\sqrt{\frac{\epsilon \tau}{4nL^2}}\right) \leq 2\log_2\left( \frac{ 2048 n L^2 \left( f(x_k)- f(x_k+ d_k \, \alpha^*) \right)}{\tau^3 \epsilon} \right)$ which brings the total number of pairwise comparisons requested to at most $\frac{8 nL}{\tau} \log\left(  \frac{f(x_0)-f(x^*)}{\epsilon/2}   \right) \log_2\left( \frac{ 2048 n L^2 \max_k \left( f(x_k)- f(x_k+ d_k \, \alpha^*) \right)}{\tau^3 \epsilon} \right)$. 

\subsection{Proof of Theorem \ref{comparisonThmUp}}
We now introduce a line search algorithm that is robust to a function comparison oracle that makes errors. Essentially, the algorithm consists of nothing more than repeatedly querying the same random pairwise comparison. This strategy applied to active learning is well known because of its simplicity and its ability to adapt to unknown noise conditions \cite{matti}. However, we mention that when used in this way, this sampling procedure is known to be sub-optimal so in practice, one may want to implement a more efficient approach like that of \cite{castroNowak}. Consider the subroutine of Figure~\ref{fig:repeatedQueries}.

\begin{figure}[h]
\centering
\fbox{\parbox[b]{3.7in}{{\underline{\bf Repeated querying subroutine}}    \\
Input: $x,y \in \R^n$, $\delta>0$\\
Initialize: $S = \emptyset$, $l=-1$\\
\textbf{do} \\
\text{\ \ \ \ } $l=l+1$\\
\text{\ \ \ \ } $\Delta_l = \sqrt{\frac{(l+1)\log(2 /\delta)}{2^{l}}}$\\
\text{\ \ \ \ } $S = S \cup \{ 2^l \text{ i.i.d. draws of  } C_f(x,y) \}$ \\
 \textbf{while} $\left| \frac{1}{2} \sum_{e_i \in S}e_i \right| - \Delta_l <0$ \\
 \textbf{return} $\text{sign} \left\{ \sum_{e_i \in S}e_i \right\}$.
 }}
\caption{Subroutine that estimates $\E \left[C_f(x,y) \right]$ by repeatedly querying the random variable.}
\label{fig:repeatedQueries}
\end{figure}

\begin{lemma} \cite{matti} \label{lemmaRepeat2}
For any $x,y \in \R^n$ with $\P \left( C_f(x,y)= \sign\{f(y)-f(x)\} \right) = p$, then with probability at least $1-\delta$ the algorithm of Figure~\ref{fig:repeatedQueries} correctly identifies the sign of $\E\left[ C_f(x,y) \right]$ and requests no more than 
\begin{align*}
\frac{\log(2 /\delta)}{4|1/2-p|^2} \log_2 \left( \frac{ \log(2 /\delta)}{4|1/2-p|^2}  \right) 
\end{align*}
pairwise comparisons. \end{lemma}

It would be convenient if we could simply apply the result of Lemma~\ref{lemmaRepeat} to the algorithm of Figure~\ref{fig:alg1d}. Unfortunately, if we do this there is no guarantee that $|f(y)- f(x)|$ is bounded below so for the case when $\kappa >1$,  it would be impossible to lower bound $|1/2-p|$ in the lemma. To account for this, we will sample at four points per iteration as opposed to just two in the noiseless algorithm to ensure that we can always lower bound $|1/2-p|$. We will see that the algorithm and analysis naturally adapts to when $\kappa =1$ or $\kappa >1$.

Consider the following modification to the algorithm of Figure~\ref{fig:alg1d}. We discuss the sampling process that takes place in $[\alpha_k, \, \alpha_k^+]$ but it is understood that the same process is repeated symmetrically in $[\alpha_k^-, \,  \alpha_k]$. We begin with the first two \texttt{while} loops. Instead of repeatedly sampling $C_f(x, x+ d \, \alpha_k^+)$ we will have two sampling procedures running in parallel that repeatedly compare $\alpha_k$ to $ \alpha_k^+$ and $\alpha_k$ to $ 2\alpha_k^+$. As soon as the repeated sampling procedure terminates for one of them we terminate the second sampling strategy and proceed with what the noiseless algorithm would do with $\alpha_k^+$ assigned to be the sampling location that finished first. Once we're out of the initial \texttt{while} loops, instead of comparing $\alpha_k$ to $ \frac{1}{2} (\alpha_k+ \alpha_k^+)$ repeatedly, we will repeatedly compare $\alpha_k$ to $\frac{1}{3} (\alpha_k + \alpha_k^+)$ and $\alpha_k$ to $\frac{2}{3} (\alpha_k + \alpha_k^+)$. Again, we will treat the location that finishes its sampling first as $ \frac{1}{2} (\alpha_k+ \alpha_k^+)$ in the noiseless algorithm. 

If we perform this procedure every iteration, then at each iteration we are guaranteed to remove at least $1/3$ the search space, as opposed to $1/2$  in the noiseless case, so we realize that the number of iterations of the robust algorithm is within a constant factor of the number of iterations of the noiseless algorithm. However, unlike the noiseless case where at most two pairwise comparisons were requested at each iteration, we must now apply Lemma~\ref{lemmeRepeat} to determine the number of pairwise comparisons that are requested per iteration.  

Intuitively, the repeated sampling procedure requests the most pairwise comparisons when the distance between the two function evaluations being compared  smallest. This corresponds to when the distance between probe points is smallest, i.e. when $\eta/2 \leq |\alpha_k - \alpha^*| \leq \eta$. By considering this worst case, we can bound the number of pairwise comparisons that are requested at any iteration.  By strong convexity $(\tau)$ we find through a straightforward calculation that $\max\left\{ | f(x+ d\, \alpha_k) - f(x + d\, \frac{2}{3}(\alpha_k+\alpha_k^{+}))|, | f(x+ d\, \alpha_k) - f(x + d\, \frac{1}{3}(\alpha_k+\alpha_k^{+}))| \right\} \geq \frac{\tau}{18} \eta^2$ for all $k$. This implies  $|1/2-p| \geq  \mu \left( \frac{\tau}{18} \eta^2 \right)^{\kappa-1}$ so that on on any given call to the repeated querying subroutine, with probability at least $1-\delta$ the subroutine requests no more than $\widetilde{O}\left( \frac{\log(1/\delta)}{\left(\tau \eta^2 \right)^{2(\kappa-1)} } \right)$ pairwise comparisons. However, because we want the total number of calls to the subroutine to hold with probability $1-\delta$, not just one, we must union bound over   $4$ pairwise comparisons per iteration times the number of iterations per line search times the number of line searches. This brings the total number of calls to the repeated query subroutine to no more than $4 \times \frac{3}{2} \log_2\left( \frac{ 256 L \max_k\left( f(x_k)- f(x_k+ d_k \, \alpha_k^*) \right)}{\tau^2 \eta^2} \right) \times \frac{4n L}{ \tau} \log\left(  \frac{f(x_0)-f(x^*)}{\eta^2 2 n L^2 / \tau}   \right) = O\left( n\frac{L}{\tau} \log^2\left( \frac{f(x_0)-f(x^*)}{n  \eta^2} \right)\right)$. If we set $\eta = \left( \frac{\epsilon \tau}{4 n L^2} \right)^{1/2}$ so that  $\E\left[f(x_K)-f(x^*) \right] \leq \epsilon$ by Theorem~\ref{coordDescentThm}, then the total number of requested pairwise comparisons does not exceed
\begin{align*}
\widetilde{O}\left( \frac{ nL}{\tau}  \left(\frac{n}{\epsilon}\right)^{2(\kappa-1)} \log^2\left(  \frac{ f(x_0)-f(x^*)  }{\epsilon}   \right)  \log(n/\delta)  \right).
\end{align*}

By finding a $T>0$ that satisfies this bound for any $\epsilon$ we see that this is equivalent to a rate of ${O} \left( n \log(n/\delta) \left( \frac{n}{T} \right)^{\frac{1}{2(\kappa-1)}} \right)$ for $\kappa>1$ and ${O} \left( \exp\left\{{-c\sqrt{\frac{T}{n \log(n/\delta)} }} \right\} \right)$ for $\kappa=1$, ignoring $\text{polylog}$ factors.

 \end{document}